 \documentclass[accepted]{uai2023} % after acceptance, for a revised
\usepackage[american]{babel}
\usepackage{natbib} % has a nice set of citation styles and commands
\usepackage{mathtools} % amsmath with fixes and additions
\usepackage{booktabs} % commands to create good-looking tables
\usepackage{tikz} % nice language for creating drawings and diagrams

\DeclareMathOperator{\ReLU}{ReLU}
\DeclareMathOperator{\Dropout}{Dropout}

\DeclareMathOperator{\Bernoulli}{Bernoulli}
\DeclareMathOperator{\GAP}{GAP}

\DeclareMathOperator{\E}{E}
\DeclareMathOperator{\Var}{Var}
\DeclareMathOperator{\train}{train}
\DeclareMathOperator{\test}{test}
\usepackage{amsthm}
\newtheorem{definition}{Definition}

\newtheorem{proposition}{Proposition}
\usepackage{cases}
\usepackage{subcaption}
\usepackage{multirow}
\newtheorem{guideline}{Guideline}
\newcommand{\ie}{\textit{i}.\textit{e}.}
\newcommand{\eg}{\textit{e}.\textit{g}.}
\usepackage{amsfonts}  % mathbb
\usepackage{bbm}  % mathbbm

\usepackage{mathabx}

\title{How to Use Dropout Correctly on Residual Networks\\with Batch Normalization}

\author[1]{Bum Jun Kim}
\author[1]{Hyeyeon Choi}
\author[1]{Hyeonah Jang}
\author[1]{Donggeon Lee}
\author[1]{\href{mailto:<swkim@postech.edu>?Subject=Your UAI 2023 paper}{Sang Woo Kim}{}}
\affil[1]{%
    Department of Electrical Engineering.\\
    Pohang University of Science and Technology\\
    Pohang, South Korea
}

\begin{document}
\maketitle

\begin{abstract}
	For the stable optimization of deep neural networks, regularization methods such as dropout and batch normalization have been used in various tasks. Nevertheless, the correct position to apply dropout has rarely been discussed, and different positions have been employed depending on the practitioners. In this study, we investigate the correct position to apply dropout. We demonstrate that for a residual network with batch normalization, applying dropout at certain positions increases the performance, whereas applying dropout at other positions decreases the performance. Based on theoretical analysis, we provide the following guideline for the correct position to apply dropout: apply one dropout after the last batch normalization but before the last weight layer in the residual branch. We provide detailed theoretical explanations to support this claim and demonstrate them through module tests. In addition, we investigate the correct position of dropout in the head that produces the final prediction. Although the current consensus is to apply dropout after global average pooling, we prove that applying dropout before global average pooling leads to a more stable output. The proposed guidelines are validated through experiments using different datasets and models.
\end{abstract}

\section{Introduction}\label{sec:intro}
Deep neural networks have demonstrated remarkable performance across a range of fields including computer vision and natural language processing. Previously, training a deep neural network using a large number of parameters was known to be difficult owing to the overfitting problem. To address this issue, several regularizers such as dropout, batch normalization (BN), and label smoothing have recently been proposed \citep{DBLP:conf/icml/IoffeS15,DBLP:journals/jmlr/SrivastavaHKSS14,DBLP:conf/cvpr/SzegedyVISW16}. They have made significant contributions to the stable optimization of deep neural networks and have been widely used in various tasks.

For the architectural design of modern neural networks, the [BN--ReLU--Weight] pipeline has been widely used, where dropout can be added. However, there remains a lack of consensus regarding the correct position for applying dropout, and practitioners have chosen different positions. For example, \citet{DBLP:conf/aaai/PhamL21,DBLP:conf/cvpr/IsolaZZE17,DBLP:journals/tits/RomeraABA18,DBLP:conf/aaai/YanXL18} applied dropout \textbf{after every BN}; however, in the studies of \citet{DBLP:journals/corr/abs-1904-03392,DBLP:conf/nips/QiYSG17,DBLP:conf/cvpr/PavlloFGA19,DBLP:conf/cvpr/0028C0019,DBLP:conf/bmvc/ZagoruykoK16,DBLP:conf/cvpr/ZhanX0OL20}, dropout was applied \textbf{after each ReLU}. Furthermore, \citet{DBLP:journals/ijon/CastroSOP021,DBLP:conf/nips/GhiasiLL18} used dropout \textbf{after the weight layers} and \citet{DBLP:conf/iclr/RaviL17,DBLP:journals/cgf/LimGK16,DBLP:journals/corr/abs-2006-15619} applied dropout \textbf{after every MaxPool layer}. Based on these practices, we highlight the need for further research to determine the correct position to apply dropout.

In fact, \citet{DBLP:conf/eccv/HeZRS16} empirically found that applying dropout at the output of the residual block decreased the performance, whereas \citet{DBLP:conf/bmvc/ZagoruykoK16} reported that applying dropout inside the residual branch improved the performance. These observations highlight the importance of selecting the correct position for dropout. That is, choosing the incorrect order of layers implies a potential performance decrease, and if dropout is placed in the correct position, a potential performance improvement can be obtained at little extra cost.

Moreover, there is a lack of theoretical analysis to determine the correct position to apply dropout. As an exception, only the study by \citet{DBLP:conf/cvpr/0028C0019} theoretically discussed the position of applying dropout. Their analysis advocated using dropout before each weight layer to harmonize the dropout and BN. However, we present the limitations and a reinterpretation of their study. For example, starting from their analysis and considering residual networks, we derive a different conclusion regarding the correct position for the dropout.

In this study, we investigate the correct position of dropout. First, we analyze the different dropout operations in the training and test phases, which are harmful to the normalization step of BN. We quantify the different behaviors of dropout in the training and test phases as an inconsistency ratio and argue that the inconsistency ratio is influenced by the order of the layers. Considering this phenomenon, we discuss the best order of layers to mitigate the inconsistency ratio. Our conclusions suggest that dropout and ReLU are permutable (Proposition \ref{pro:relu}); using dropout before the weight layer resolves the inconsistency ratio under certain weight conditions (Proposition \ref{pro:postpre}); and residual blocks mitigate the inconsistency ratio better than non-residual blocks for certain positions (Propositions \ref{pro:zeroth} and \ref{pro:lth}). Based on these analyses, we propose applying one dropout after the last BN but before the last weight layer in the residual branch (Guideline \ref{gui:my}).

In addition, we analyze the use of dropout in the head that outputs the final prediction. Although the current consensus is to apply dropout after global average pooling (GAP), we prove that using dropout before GAP leads to a more stable output (Proposition \ref{pro:head}). Based on this analysis, we propose the use of dropout before GAP (Guideline \ref{gui:head}).

The validities of Guidelines \ref{gui:my} and \ref{gui:head} are verified through experiments on different datasets including CIFAR-$\{10,\ 100\}$, Caltech-101, Oxford IIIT-Pet, and ImageNet. We observed that the performance of the model improved with dropout when the guidelines were followed.

\section{Theoretical Analysis}\label{sec:theory}
\paragraph{Notation} In this paper, we use the notations $\E[x_i]$ and $\Var[x_i]$ to denote the mean and variance of the $i$th element $x_i$ of vector $\mathbf{x}$ over the mini-batch. To represent $\Var[x_i]$ for an arbitrary index $i$, we use the abbreviation $\Var[\mathbf{x}]$.
\subsection{Background}
Dropout is an operation that randomly drops certain features in the target layer during training \citep{DBLP:journals/jmlr/SrivastavaHKSS14}. First, we define the Dropout operation as follows.
\begin{definition}
	Operation \textit{Dropout} with a keep probability $p \in (0,\ 1)$ is defined as follows:
	\begin{align}
		\Dropout_{\train}(\mathbf{x}) & \coloneqq \frac{1}{p} \mathbf{M} \mathbf{x}, \\
		\Dropout_{\test}(\mathbf{x})  & \coloneqq  \mathbf{x},
	\end{align}
	where $\mathbf{x}$ is an $n$-dimensional vector and $\mathbf{M}$ is an $n \times n$ diagonal matrix with $m_{ij} = 0$ for $i \neq j$ and $m_{ij} \sim \Bernoulli(p)$ for $i=j$.
\end{definition}
According to the Bernoulli distribution, $m_{i,i}$ is either one with keep probability $p$ or zero with drop probability $1-p$, and is independent of $\mathbf{x}$. Thus, we have $\E[m_{i,i}^2]=\E[m_{i,i}]=p$ and $\E[m_{i,i}m_{j,j}]=p^2$ for $i \neq j$. This property ensures mean consistency in the training and test phases, \ie, $\E[\frac{1}{p} \mathbf{M} \mathbf{x}] = \E[\mathbf{x}]$.

However, dropout does not provide variance consistency in the training and test phases. To investigate this phenomenon, we introduce the following inconsistency ratio:
\begin{definition}
	Let $f(\mathbf{x})$ be the output feature of an operation $f$. The \textit{inconsistency ratio} $\Delta(f(\mathbf{x}))$ is defined as the ratio of the variance of $f(\mathbf{x})$ between the training and test phases.
	\begin{align}
		\Delta(f(\mathbf{x})) \coloneqq \frac{\Var[f_{\test}(\mathbf{x})]}{\Var[f_{\train}(\mathbf{x})]}.
	\end{align}
\end{definition}
For example, $\Delta(f(\mathbf{x}))=0.5$ indicates that the variance of $f(\mathbf{x})$ during the training phase is twice as large as that during the test phase. To obtain variance consistency, we should achieve $\Delta(f(\mathbf{x}))=1$.

\citet{DBLP:conf/cvpr/0028C0019} state that the use of dropout yields variance inconsistency in a neural network. During the training phase,
\begin{align}
	 & \Var[\Dropout_{\train}(\mathbf{x})]                                       \\
	 & = \E[ \frac{1}{p^2} m_{i,i}^2 x_i^2 ] - (\E[ \frac{1}{p} m_{i,i} x_i ])^2 \\
	 & = \frac{1}{p} \Var[x_i] + \frac{1-p}{p} (\E[x_i])^2,
\end{align}
which is greater than $\Var[\Dropout_{\test}(\mathbf{x})]=\Var[x_i]$ for $p<1$. Thus, $\Delta(\Dropout(\mathbf{x}))<1$.

The variance inconsistency of dropout causes a problem when we use dropout with BN. Although subsequent BN anticipates receiving the same mean and variance during the training and test phases, the variance inconsistency of dropout provides different variances to BN during the training and test phases. For example, consider an input feature $\mathbf{h}$ to BN where $\Var[\mathbf{h}_{\train}]=10$, $\Var[\mathbf{h}_{\test}]=2$, $\E[\mathbf{h}_{\train}]=0$, and $\E[\mathbf{h}_{\test}]=0$. The normalization step of BN uses the mean and variance of the training phase to produce $\frac{\mathbf{h}}{\sqrt{10}}$, which is also used in the test phase because BN assumes the same mean and variance. After the normalization step, we obtain $\Var[\frac{\mathbf{h}_{\train}}{\sqrt{10}}]=1$ during the training phase. However, during the test phase, the BN receives a feature with a different variance, resulting in $\Var[\frac{\mathbf{h}_{\test}}{\sqrt{10}}]=0.2$. Thus, the variance inconsistency breaks the consistent behavior of the subsequent BN during the training and test phases. This phenomenon explains the decrease in performance when dropout and BN are used simultaneously.

\subsection{Order of Operations}
A modern neural network is composed of numerous operations such as ReLU, weight layer, BN, and skip connection, whose output feature map is a potential position for applying dropout. Here, we claim that the position of applying dropout influences the inconsistency ratio. The goal of this study is to investigate the best position for applying dropout that offers a $\Delta(f(\mathbf{x}))$ close to one, which harmonizes the dropout with BN. First, we discuss the order of the operations.

\paragraph{Order of Dropout and ReLU} Some practitioners have applied dropout before ReLU, whereas others have applied dropout after ReLU (Section \ref{sec:intro}). Here, we claim that the influence of the order of ReLU and dropout is insignificant.

\begin{proposition} \label{pro:relu}
	ReLU and dropout operations are permutable:
	\begin{align}
		\ReLU(\Dropout(\mathbf{x})) = \Dropout(\ReLU(\mathbf{x})). \label{eq:commute}
	\end{align}
\end{proposition}

\begin{proof}
	First, during the test phase, dropout operates as an identity function that satisfies Eq. \ref{eq:commute}. Secondly, we claim that, even during the training phase, the influence of the order of ReLU and dropout is insignificant. Consider that we sampled matrix $\mathbf{M}$ to denote $\Dropout_{\train}^\mathbf{M}(\mathbf{x})$. Because matrix $\mathbf{M}$ is a diagonal matrix, the $i$th element of vector $\ReLU(\Dropout_{\train}^\mathbf{M}(\mathbf{x}))$ can be written as
	\begin{align}
		[\ReLU(\Dropout_{\train}^\mathbf{M}(\mathbf{x}))]_i = \ReLU \Bigl( \frac{1}{p} m_{i,i} x_i \Bigr).
	\end{align}
	Here, the coefficient $m_{i,i}/p$ is a non-negative scalar. For $\ReLU(x)=\max(0,\ x)$, we know that $\ReLU(kx)=k\ReLU(x)$ for a non-negative scalar $k$. Thus, we obtain
	\begin{align}
		\ReLU \Bigl( \frac{1}{p} m_{i,i} x_i \Bigr) & = \frac{1}{p} m_{i,i} \ReLU(x_i)                       \\
		                                            & = [\Dropout_{\train}^\mathbf{M}(\ReLU(\mathbf{x}))]_i.
	\end{align}
	Therefore, we conclude that $\ReLU(\Dropout_{\train}^\mathbf{M}(\mathbf{x})) = \Dropout_{\train}^\mathbf{M}(\ReLU(\mathbf{x}))$.
\end{proof}
In summary, the order in which dropout and ReLU operations are applied to vector $\mathbf{x}$ does not influence the result. In the remainder of this paper, we do not consider applying dropout before ReLU unless specified otherwise.

However, commutativity with dropout does not hold for other operations such as weight layer and BN. We further investigate the effects of the order of these operations.

\paragraph{Order of Dropout and Weight}
We refer to dropout before the weight layer as \textit{PreDropout}. For PreDropout, we have $\mathbf{W}\Dropout_{\train}(\mathbf{x}) = \frac{1}{p} \mathbf{W} \mathbf{M} \mathbf{x}$ and can interpret the two operations using another weight $\mathbf{W} \mathbf{M}/p$. Similarly, for the PostDropout order, we write $\Dropout_{\train}(\mathbf{W}\mathbf{x}) = \frac{1}{p} \mathbf{M} \mathbf{W} \mathbf{x}$. The difference between PreDropout and PostDropout occurs because $\mathbf{W}\mathbf{M} \neq \mathbf{M}\mathbf{W}$ for $p < 1$. These matrices can be represented as
\begin{align*}
	\mathbf{W}\mathbf{M} & = \begin{bmatrix}
		                         \vert      & \vert  & \vert      \\
		                         m_{1,1}w_1 & \cdots & m_{n,n}w_n \\
		                         \vert      & \vert  & \vert
	                         \end{bmatrix},     \\
	\mathbf{M}\mathbf{W} & =\begin{bmatrix}
		                        \text{---} & m_{1,1} w_1 & \text{---} \\
		                        \text{---} & \cdots      & \text{---} \\
		                        \text{---} & m_{m,m} w_m & \text{---}
	                        \end{bmatrix}.
\end{align*}
The diagonal element $m_{i,i}$ is either zero or one. Thus, PreDropout is equivalent to dropping \textit{columns} in the weight matrix $\mathbf{W}$ with $1/p$ constant scaling, whereas PostDropout is equivalent to dropping \textit{rows} in the weight matrix $\mathbf{W}$ with $1/p$ constant scaling. Thus, the characteristics of PreDropout and PostDropout differ for $p<1$.

The question then arises as to which is more effective in reducing variance inconsistency. \citet{DBLP:conf/cvpr/0028C0019} suggest that for the PreDropout order, increasing the width alleviates variance inconsistency, assuming a certain condition on weight. However, we find that increasing the width does not solve the variance inconsistency for other weight conditions such as He initialization \citep{DBLP:conf/iccv/HeZRS15}. Although \citet{DBLP:conf/cvpr/0028C0019} emphasized increasing the width, we focus more on the weight condition. Our reinterpretation of their study is as follows.
\begin{proposition} \label{pro:postpre}
	PreDropout exhibits less variance inconsistency than PostDropout
	\begin{align}
		\Delta(\underbrace{\Dropout(\mathbf{W}\mathbf{x})}_{\text{PostDropout}}) < \Delta(\underbrace{\mathbf{W}\Dropout(\mathbf{x})}_{\text{PreDropout}}) < 1,
	\end{align}
	where the first inequality holds if and only if $\sum_{j=1}^n \sum_{k \neq j}^n w_{i,j}w_{i,k} \E[x_j x_k ] > 0$.
\end{proposition}
According to Proposition \ref{pro:postpre}, the advantage of PreDropout depends on the weight condition. For example, if the weight has a nonzero mean and $\mathbf{x}$ comes from ReLU output, the condition holds and PreDropout is advantageous. However, for zero-mean weight, the inconsistency ratios of PostDropout and PreDropout can be indistinguishable. A detailed proof can be found in the Appendix.

\begin{figure}[t!]
	\centering
	\includegraphics[width=0.89\linewidth]{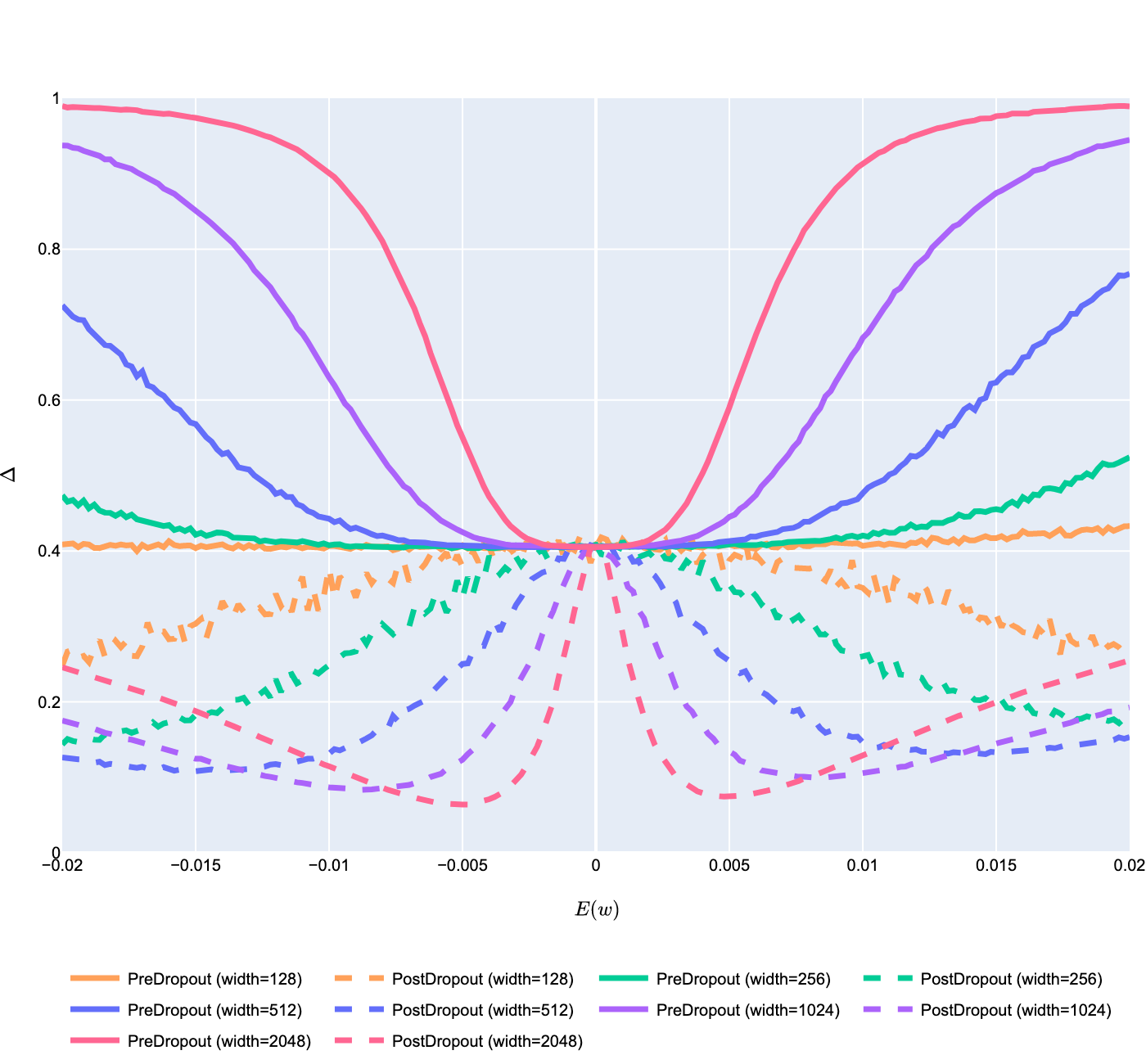}
	\caption{Empirical validation on Proposition \ref{pro:postpre}. We observed $\Delta(\Dropout(\mathbf{W}\mathbf{x})) < \Delta(\mathbf{W}\Dropout(\mathbf{x})) < 1$ for $\E[\mathbf{W}] \neq 0$.}
	\label{fig:pro2}
\end{figure}

\paragraph{Empirical Observation} We measured the two inconsistency ratios, $\Delta(\Dropout(\mathbf{W}\mathbf{x}))$ and $\Delta(\mathbf{W}\Dropout(\mathbf{x}))$. The [BN--ReLU--Weight--BN--ReLU] pipeline, which is commonly deployed in the residual branch, produced output $\mathbf{x}$. The input to the pipeline was sampled from $\mathcal{N}(0,\ 1)$ with a mini-batch size of $10^5$. We tested five cases of width $n$ from $\{128,\ 256,\ 512,\ 1024,\ 2048\}$. We used a keep probability of 0.5 for dropout. For the weight condition, we set $\Var[\mathbf{W}]=2/n$, similar to the He initialization but varied $\E[\mathbf{W}]$.

The results are summarized in Figure \ref{fig:pro2}. We observed that $\Delta(\Dropout(\mathbf{W}\mathbf{x})) < \Delta(\mathbf{W}\Dropout(\mathbf{x})) < 1$ for $\E[\mathbf{W}] \neq 0$. The advantage of using PreDropout is visible in both $\E[\mathbf{W}] > 0$ and $\E[\mathbf{W}] < 0$. This is because $\E[x_j x_k] \geq 0$ for $\mathbf{x}$ from the ReLU output and $\sum_{j=1}^n \sum_{k \neq j}^n w_{i,j}w_{i,k} \approx n(n-1)(\E[\mathbf{W}])^2 > 0$ for large $\left\lvert \E[\mathbf{W}] \right\rvert$. A large width has little effect if $\E[\mathbf{W}]=0$. Thus, we found that the advantage of PreDropout requires $\E[\mathbf{W}] \neq 0$ and intensifies as the width increases.

Therefore, the advantage of the PreDropout is dependent on the weight condition. For example, if the weight is polarized to a large $\left\lvert \E[\mathbf{W}] \right\rvert$ through training, the accumulation of its products becomes positive, allowing us to enjoy the advantage of PreDropout. \citet{DBLP:conf/cvpr/0028C0019} empirically observed that the trained weight satisfies a certain condition to advocate PreDropout. We conjecture that the condition holds and validate the superior performance of PreDropout over PostDropout through experiments (Section \ref{sec:exp}).

However, we later demonstrate cases where neither PreDropout nor PostDropout improves performance. Rather than comparing PreDropout and PostDropout, we find that the properties of residual networks have a greater influence on the alleviation of variance inconsistency.

\subsection{Dropout in Residual Block}
A residual network is composed of residual blocks, which consist of a residual and skip branch. PreResNet, also known as ResNetV2, is a variant that applies BN first in the residual branch \citep{DBLP:conf/eccv/HeZRS16}. In this section, we provide an analysis of PreResNet, which can be extended to other variants of residual networks such as ResNetV1 \citep{DBLP:conf/cvpr/HeZRS16}. We examine eight possible positions to apply dropout, labeled P0--P7 (Figure \ref{fig:residualblock}).

First, we consider applying dropout at P1. As mentioned earlier, the use of dropout causes inconsistency in the input variance of the next BN. If we apply dropout at P1, it directly influences the input variance of the first BN of the $l$th residual block. Similarly, applying dropout at one of (P2, P3, or P4) results in an inconsistency in the input variance of the second BN of the $l$th residual block. Furthermore, applying dropout at P0 causes variance inconsistency in the first BN of the $l$th residual block.

\begin{figure}[t!]
	\centering
	\includegraphics[width=0.81\linewidth]{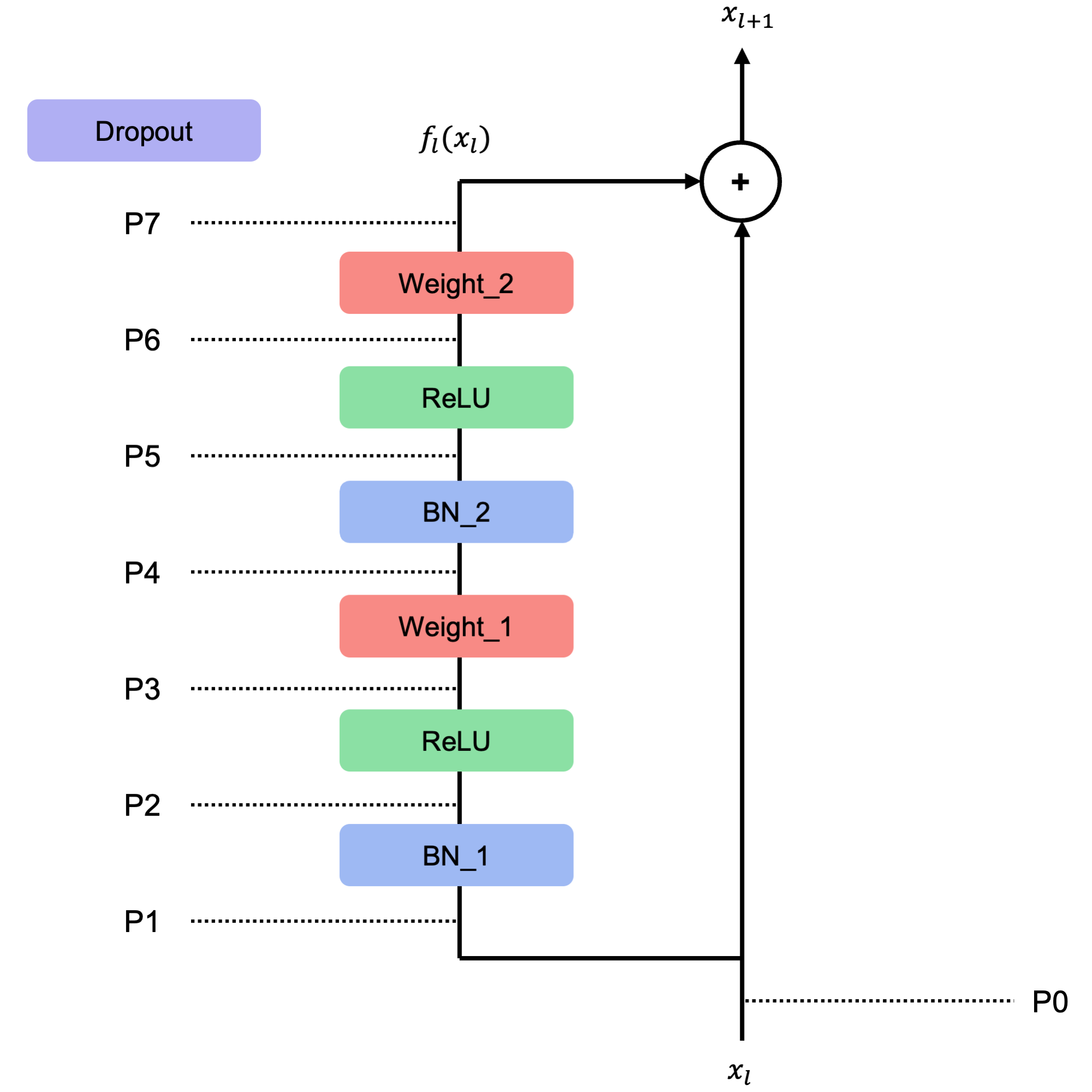}
	\caption{Residual block of PreResNet.}
	\label{fig:residualblock}
\end{figure}

However, applying dropout at one of (P5, P6, or P7) has distinct characteristics. The use of dropout at one of (P5, P6, or P7) causes inconsistency in the input variance of the BN at the next residual block, \ie, the first BN of the ($l+1$)th residual block. Because the output of the residual branch is merged with the skip connection, the variance inconsistency in the next residual block behaves differently compared to the other positions. From this observation, we investigate the inconsistency ratio of dropout at (P5, P6, or P7) in detail.

First, we analyze the output $f_l(\mathbf{x}_l)$ of the [BN--ReLU--Weight--BN--ReLU--Dropout--Weight] pipeline from the input feature map $\mathbf{x}_l$ of $l$th block, considering dropout at P6. Note that BN outputs $\gamma \hat{x} + \beta$ from the normalized feature $\hat{x}$. For the $l$th residual block, let the output of the second BN be $\mathbf{z}_l \sim \mathcal{N}(0,\ \gamma_l^2)$.\footnote{At initialization, BN has $\gamma=1$ and $\beta=0$. During training, although $\gamma$ becomes a specific value, $\beta$ stays close to zero. We conjecture that $\beta$ near zero is advantageous to preserving zero-centered ReLU input. From this observation, we allow a degree of freedom in $\gamma$; however, we use $\beta=0$.} This feature map passes through ReLU, weight, and skip connection. First, we know
\begin{align}
	\E[\ReLU(\mathbf{z}_l)]   & =\frac{1}{\sqrt{2 \pi}}\gamma_l,   \\
	\Var[\ReLU(\mathbf{z}_l)] & = \frac{\pi -1 }{2 \pi}\gamma_l^2.
\end{align}
See the Appendix for details of the above equations. Now, we apply dropout. During the training phase,
\begin{align}
	\E[\Dropout_{\train}(\ReLU(\mathbf{z}_l))]   & =\frac{1}{\sqrt{2 \pi}}\gamma_l,      \\
	\Var[\Dropout_{\train}(\ReLU(\mathbf{z}_l))] & = \frac{\pi/p - 1}{2 \pi} \gamma_l^2.
\end{align}
Finally, we apply weight that is initialized from He initialization $\mathcal{N}(0,\ 2/n)$. Then, we have
\begin{align}
	\E[\mathbf{W}\Dropout_{\train}(\ReLU(\mathbf{z}_l))]   & = 0,                                \\
	\Var[\mathbf{W}\Dropout_{\train}(\ReLU(\mathbf{z}_l))] & = \frac{\pi/p - 1}{\pi} \gamma_l^2,
\end{align}
which represents the variance of the residual branch $f_l(\mathbf{x}_l)$:
\begin{align}
	\Var[f_{l,\train}(\mathbf{x}_l)] & = \frac{\pi/p - 1}{\pi} \gamma_l^2, \\
	\Var[f_{l,\test}(\mathbf{x}_l)]  & = \frac{\pi - 1}{\pi} \gamma_l^2.
\end{align}
The same result can be obtained when applying dropout at P5 (Proposition \ref{pro:relu}) and P7 (Proposition \ref{pro:postpre} for the zero-mean weight). Now, consider two choices for the building blocks: non-residual block $f_l(\mathbf{x}_l)$ and residual block $\mathbf{x}_l + f_l(\mathbf{x}_l)$. We begin by investigating the case where $l=0$.

\begin{proposition} \label{pro:zeroth}
	For the $0$th very first block, choosing a residual block alleviates variance inconsistency from dropout when compared with a non-residual block
	\begin{align}
		\Delta(\underbrace{f_0(\mathbf{x}_0)}_{\text{Non-residual}}) < \Delta(\underbrace{\mathbf{x}_0 + f_0(\mathbf{x}_0)}_{\text{Residual}}) < 1,
	\end{align}
	if we apply dropout at one of (P5, P6, or P7) in PreResNet.
\end{proposition}
\begin{proof}
	Note that for $x>0$, $y>0$, and $c>0$, if $\frac{x}{y}<1$, then $\frac{x}{y}<\frac{x+c}{y+c}$. Using this inequality, we obtain
	\begin{align}
		\Delta(f_0(\mathbf{x}_0)) & = \frac{\Var[f_{0,\test}(\mathbf{x}_0)]}{\Var[f_{0,\train}(\mathbf{x}_0)]}                                           \\
		                          & < \frac{\Var[\mathbf{x}_0] + \Var[f_{0,\test}(\mathbf{x}_0)]}{\Var[\mathbf{x}_0] + \Var[f_{0,\train}(\mathbf{x}_0)]} \\
		                          & = \Delta(\mathbf{x}_0 + f_0(\mathbf{x}_0)) < 1.
	\end{align}
	The advantage of choosing a residual block appears when the skip connection is located after dropout but before the subsequent BN. Thus, applying dropout at one of (P5, P6, or P7) alleviates variance inconsistency. Others, such as (P2, P3, or P4) correspond to non-residual blocks and do not alleviate variance inconsistency.
\end{proof}

The above derivation exploits the fact that, for the very first block, the input feature map $\mathbf{x}_0$ exhibits no variance inconsistency. However, when we choose a residual block, subsequent blocks receive the input feature map $\mathbf{x}_l$, which exhibits variance inconsistency due to dropout. Nonetheless, even in this scenario, choosing a residual block is still advantageous for reducing variance inconsistency.
\begin{proposition} \label{pro:lth}
	For the $l$th block, if all $l^\prime$th blocks for $l^\prime < l$ are residual blocks, then choosing a residual block alleviates variance inconsistency from dropout when compared to a non-residual block
	\begin{align}
		\Delta(\underbrace{f_l(\mathbf{x}_l)}_{\text{Non-residual}}) < \Delta(\underbrace{\mathbf{x}_l + f_l(\mathbf{x}_l)}_{\text{Residual}}) < 1,
	\end{align}
	if we apply dropout at one of (P5, P6, or P7) in PreResNet.
\end{proposition}
\begin{proof}
	The skip connection adds the result of the residual branch as $\mathbf{x}_{l+1} = \mathbf{x}_l + f_l(\mathbf{x}_l)$. As \citet{DBLP:conf/nips/DeS20,DBLP:conf/iclr/BrockDS21} describe, the residual block accumulates its variance:
	\begin{align}
		\Var[\mathbf{x}_{l+1}] = \Var[\mathbf{x}_l] + \Var[f_l(\mathbf{x}_l)].
	\end{align}
	Thus, $\mathbf{x}_l$ is the accumulation of the residual branches from $0$ to $l-1$. For the training and test phases,
	\begin{align}
		\Var[\mathbf{x}_{l,\train}] & = \Var[\mathbf{x}_0] + \frac{\pi/p - 1}{\pi} \sum_{i=0}^{l-1} \gamma_i^2, \\
		\Var[\mathbf{x}_{l,\test}]  & = \Var[\mathbf{x}_0] + \frac{\pi - 1}{\pi} \sum_{i=0}^{l-1} \gamma_i^2.
	\end{align}
	First, if we choose a non-residual block for the $l$th block, we obtain $f_l(\mathbf{x}_l)$ and its inconsistency ratio as
	\begin{align}
		\Delta(f_l(\mathbf{x}_l)) & = \frac{\Var[f_{l,\test}(\mathbf{x}_l)]}{\Var[f_{l,\train}(\mathbf{x}_l)]}                               \\
		                          & = \frac{\frac{\pi - 1}{\pi} \gamma_l^2}{\frac{\pi/p - 1}{\pi} \gamma_l^2} = \frac{\pi - 1}{\pi / p - 1}.
	\end{align}
	Second, if we choose a residual block for the $l$th block, we obtain $\mathbf{x}_l + f_l(\mathbf{x}_l)$ and its inconsistency ratio as
	\begin{align}
		\Delta(\mathbf{x}_l + f_l(\mathbf{x}_l)) & = \Delta(\mathbf{x}_{l+1}) = \frac{\Var[\mathbf{x}_{l+1, \test}]}{\Var[\mathbf{x}_{l+1, \train}]}                                                      \\
		                                         & = \frac{\Var[\mathbf{x}_0] + \frac{\pi - 1}{\pi} \sum_{i=0}^{l} \gamma_i^2}{\Var[\mathbf{x}_0] + \frac{\pi/p - 1}{\pi} \sum_{i=0}^{l} \gamma_i^2} < 1.
	\end{align}
	Finally, it is known that for $x>0$, $y>0$, and $c>0$, if $\frac{x}{y}<1$, then $\frac{x}{y}<\frac{x+c}{y+c}$. Using this inequality, we obtain
	\begin{align}
		\Delta(\mathbf{x}_l + f_l(\mathbf{x}_l)) & > \frac{ \frac{\pi - 1}{\pi} \sum_{i=0}^{l} \gamma_i^2}{\frac{\pi/p - 1}{\pi} \sum_{i=0}^{l} \gamma_i^2} = \frac{\pi - 1}{\pi / p -1} \\
		                                         & = \Delta(f_l(\mathbf{x}_l)),
	\end{align}
	which concludes the proof of this proposition.
\end{proof}
The difference between the two inconsistency ratios is due to $\Var[\mathbf{x}_0]$. Next, we empirically test the effect of $\Var[\mathbf{x}_0]$.

\begin{figure}[t!]
	\centering
	\includegraphics[width=0.89\linewidth]{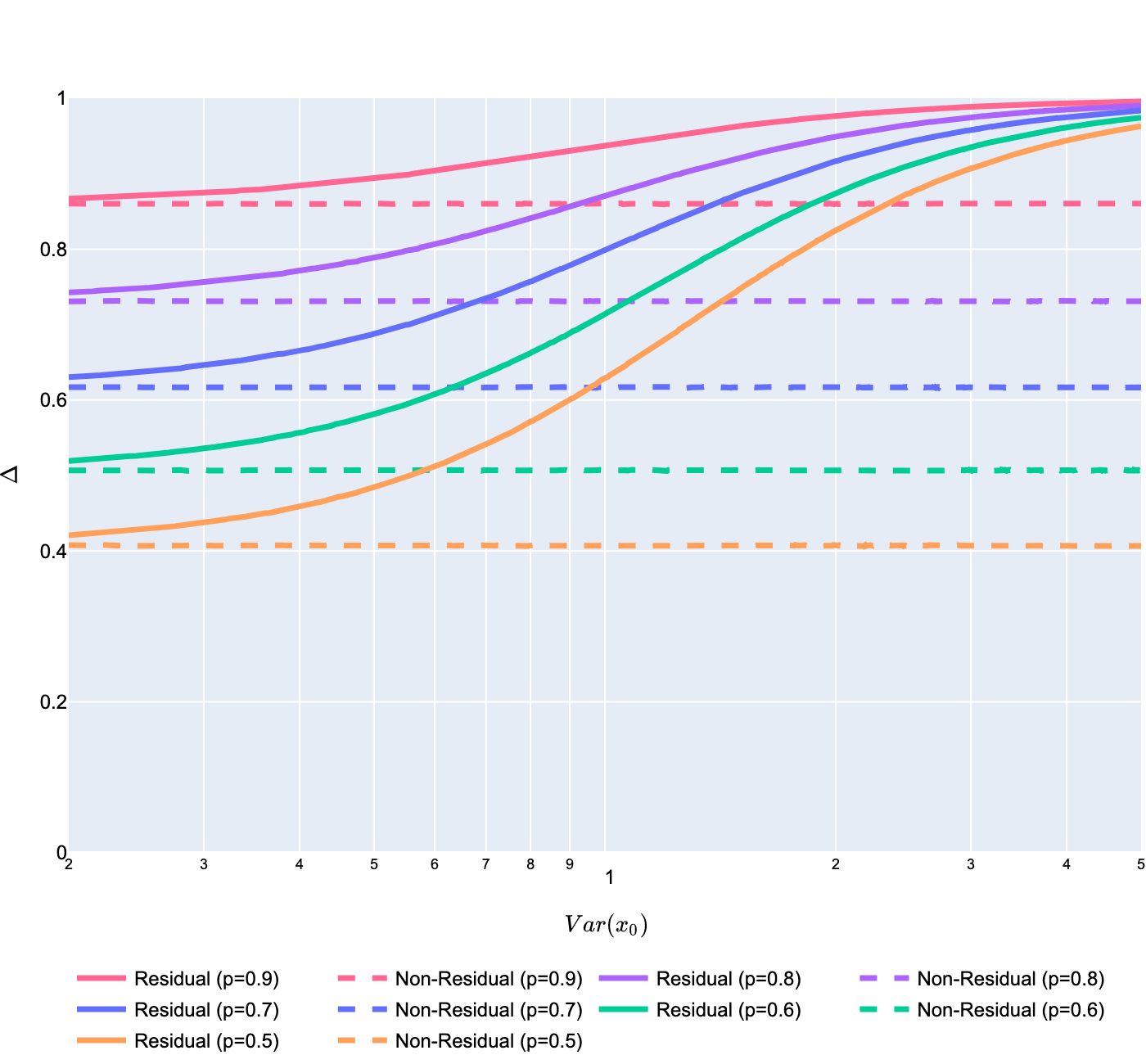}
	\caption{Empirical validation on Propositions \ref{pro:zeroth} and \ref{pro:lth}.}
	\label{fig:pro3}
\end{figure}

\paragraph{Empirical Observation} We measured the two inconsistency ratios, $\Delta(f_l(\mathbf{x}_l))$ and $\Delta(\mathbf{x}_l + f_l(\mathbf{x}_l))$. We tested five different keep probabilities $\{0.5,\ 0.6,\ 0.7,\ 0.8,\ 0.9\}$ for dropout. We applied dropout at P6 to construct the [BN--ReLU--Weight--BN--ReLU--Dropout--Weight] pipeline. We used a mini-batch size of $10^5$ with a width $n$ of 128. We initialized the weights using the He initialization $\mathcal{N}(0,\ 2/n)$. We varied $\Var[\mathbf{x}_0]$ and observed two inconsistency ratios.

The results are summarized in Figure \ref{fig:pro3}. We observed that $\Delta(f_l(\mathbf{x}_l)) < \Delta(\mathbf{x}_l + f_l(\mathbf{x}_l)) < 1$. Note that we do not say that we should achieve $\Var[\mathbf{x}_0] \rightarrow \infty$; our claim is that as long as $\Var[\mathbf{x}_0]$ is nonzero, $\Delta(\mathbf{x}_l + f_l(\mathbf{x}_l))$ obtains a gain closer to one compared to the non-residual block.

Finally, we discard P7, which corresponds to PostDropout (Proposition \ref{pro:postpre}). In summary, based on Propositions \ref{pro:relu}--\ref{pro:lth} and the above analyses, we conclude with the following guideline.
\begin{guideline} \label{gui:my}
	For each residual block of PreResNet, apply one dropout after the last BN but before the last weight layer, \eg, at P5 or P6.
\end{guideline}

\subsection{Dropout in Head}
So far, we have discussed the use of dropout in residual blocks. Additionally, we consider applying dropout in the \textit{head}, which takes the output of the last residual block as the input and outputs the final prediction. Indeed, \citet{DBLP:conf/nips/BelloFDCSLSZ21} observed improved performance when applying dropout after the GAP but before the fully connected layer. This practice has been adopted in several neural networks such as MobileNetV2, EfficientNet, EfficientNetV2, MnasNet, NASNet, and Inception-v4 \citep{DBLP:conf/cvpr/SandlerHZZC18,DBLP:conf/icml/TanL19,DBLP:conf/icml/TanL21,DBLP:conf/cvpr/TanCPVSHL19,DBLP:conf/cvpr/ZophVSL18,DBLP:conf/aaai/SzegedyIVA17}.

In light of this practice, we theoretically investigate the best position in the head to apply dropout. We examine seven possible positions to apply dropout, labeled H1 to H7 (Figure \ref{fig:head}). First, owing to the presence of BN in the head, the use of dropout at H1 and H2 should be avoided. For the remaining positions, because there is no subsequent BN, we do not discuss the inconsistency ratio further.

For the head, we now emphasize preventing dropout from resulting in an unstable output. The cross-entropy loss with a one-hot encoded label is $-\log{\hat{y}_c}$, where $\hat{y}_c$ is the probability on the correct class. For example, applying dropout at H6 or H7 directly drops the predictions, which can result in a lower $\hat{y}_c$ and a significantly larger loss. This large loss can occur even with a correct prediction, resulting in an unstable gradient descent. To obtain a stable loss, it is favorable to have a small variance at the output of the head, thus avoiding H6 and H7. Therefore, we are left with (H3, H4, or H5). Existing practices have preferred to apply dropout at H5; however, we claim that applying dropout at H3 or H4 results in a smaller variance and thus is more advantageous than H5.

\begin{figure}[t!]
	\centering
	\includegraphics[width=0.53\linewidth]{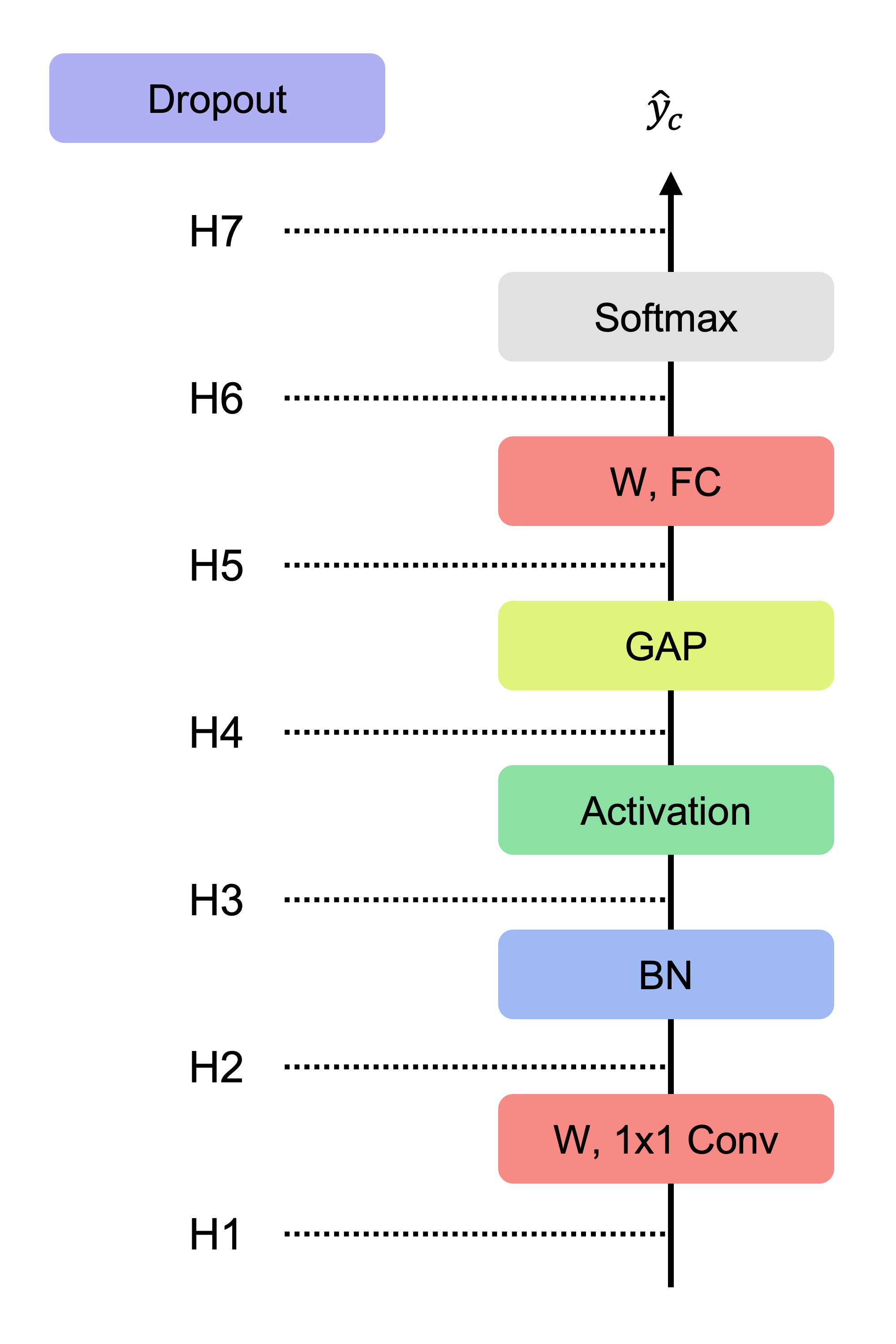}
	\caption{Seven possible positions in head to apply dropout. The composition of the head can vary depending on the model; we illustrate the head commonly deployed in models such as MobileNetV2 and EfficientNet.}
	\label{fig:head}
\end{figure}

\begin{proposition} \label{pro:head}
	Applying dropout before GAP (H4) in the head exhibits less variance compared to after GAP (H5):
	\begin{align}
		 & \Var[\underbrace{\GAP(\Dropout_{\train}(\mathbf{x}))}_{\text{H4}}]    \\
		 & < \Var[\underbrace{\Dropout_{\train}(\GAP(\mathbf{x}))}_{\text{H5}}].
	\end{align}
\end{proposition}
The difference arises from the fact that dropout before the GAP masks each element of the feature map, whereas dropout after the GAP masks each channel of the feature map. A detailed proof can be found in the Appendix.

\begin{table*}[ht!]
	\centering
	\caption{Test accuracy on CIFAR dataset. All accuracies in this paper are expressed in percentage units. The difference from baseline performance is presented to the right. $^*$ indicates applying dropout following Guideline \ref{gui:my}.}\label{tab:cifar}
	\begin{tabular}{l|rr|rr|rr|rr}
		\toprule % from booktabs package
		           & \multicolumn{4}{c|}{CIFAR-10}     & \multicolumn{4}{c}{CIFAR-100}                                                                                                                                                                                \\
		           & \multicolumn{2}{c|}{PreResNet-50} & \multicolumn{2}{c|}{PreResNet-110} & \multicolumn{2}{c|}{PreResNet-50} & \multicolumn{2}{c}{PreResNet-110}                                                                                                   \\
		           & Accuracy                          & Difference                         & Accuracy                          & Difference                        & Accuracy         & Difference                  & Accuracy         & Difference                  \\
		\midrule
		No Dropout & 93.6633                           & -                                  & 94.0300                           & -                                 & 71.3900          & -                           & 73.5367          & -                           \\
		\midrule
		P0         & 84.2500                           & (\textcolor{red}{-9.4133})         & 70.9433                           & (\textcolor{red}{-23.0867})       & 52.6200          & (\textcolor{red}{-18.7700}) & 18.3467          & (\textcolor{red}{-55.1900}) \\
		P1         & 92.9167                           & (\textcolor{red}{-0.7467})         & 93.5467                           & (\textcolor{red}{-0.4833})        & 70.2000          & (\textcolor{red}{-1.1900})  & 71.2733          & (\textcolor{red}{-2.2633})  \\
		P2         & 93.1933                           & (\textcolor{red}{-0.4700})         & 93.6867                           & (\textcolor{red}{-0.3433})        & 71.1433          & (\textcolor{red}{-0.2467})  & 72.4133          & (\textcolor{red}{-1.1233})  \\
		P3         & 93.5833                           & (\textcolor{red}{-0.0800})         & 93.8133                           & (\textcolor{red}{-0.2167})        & 71.0767          & (\textcolor{red}{-0.3133})  & 72.3967          & (\textcolor{red}{-1.1400})  \\
		P4         & 93.2167                           & (\textcolor{red}{-0.4467})         & 93.8933                           & (\textcolor{red}{-0.1367})        & 70.4800          & (\textcolor{red}{-0.9100})  & 72.1400          & (\textcolor{red}{-1.3967})  \\
		P5$^*$     & \textbf{93.8333}                  & (\textcolor{blue}{+0.1700})        & \textbf{94.4367}                  & (\textcolor{blue}{+0.4067})       & 72.3633          & (\textcolor{blue}{+0.9733}) & 73.6300          & (\textcolor{blue}{+0.0933}) \\
		P6$^*$     & 93.6767                           & (\textcolor{blue}{+0.0133})        & 94.2200                           & (\textcolor{blue}{+0.1900})       & \textbf{72.4267} & (\textcolor{blue}{+1.0367}) & \textbf{73.9800} & (\textcolor{blue}{+0.4433}) \\
		P7         & 93.7800                           & (\textcolor{blue}{+0.1167})        & 94.2667                           & (\textcolor{blue}{+0.2367})       & 72.0833          & (\textcolor{blue}{+0.6933}) & 73.5633          & (\textcolor{blue}{+0.0267}) \\
		\bottomrule
	\end{tabular}
\end{table*}

\begin{table*}[ht!]
	\centering
	\caption{Experimental results on other conditions.}
	\label{tab:condition}
	\begin{tabular}{l|rr|rr|rr|rr}
		\toprule
		                       & \multicolumn{2}{c|}{Weight Decay $10^{-3}$} & \multicolumn{2}{c|}{Weight Decay $10^{-5}$} & \multicolumn{2}{c|}{Bottleneck} & \multicolumn{2}{c}{ELU}                                                                                       \\
		                       & Accuracy                                    & Difference                                  & Accuracy                        & Difference                  & Accuracy & Difference                  & Accuracy & Difference                  \\
		\midrule
		No Dropout             & 93.2100                                     & -                                           & 92.3733                         & -                           & 93.7667  & -                           & 92.6100  & -                           \\
		\midrule
		Guideline \ref{gui:my} & 93.3567                                     & (\textcolor{blue}{+0.1467})                 & 92.8133                         & (\textcolor{blue}{+0.4400}) & 93.7933  & (\textcolor{blue}{+0.0267}) & 92.7433  & (\textcolor{blue}{+0.1333}) \\
		\bottomrule
	\end{tabular}
\end{table*}

In summary, we conclude with the following guideline:
\begin{guideline} \label{gui:head}
	In the head, apply one dropout after the BN but before the GAP layer, \eg, at H3 or H4.
\end{guideline}

\section{Experiments}\label{sec:exp}

\begin{table*}[ht!]
	\centering
	\caption{Experimental results using the ResNetV1.}
	\label{tab:resnetv1}
	\begin{tabular}{l|rr|rr|rr|rr}
		\toprule
		                       & \multicolumn{4}{c|}{CIFAR-10}    & \multicolumn{4}{c}{CIFAR-100}                                                                                                                                                             \\
		                       & \multicolumn{2}{c|}{ResNetV1-50} & \multicolumn{2}{c|}{ResNetV1-110} & \multicolumn{2}{c|}{ResNetV1-50} & \multicolumn{2}{c}{ResNetV1-110}                                                                                   \\
		                       & Accuracy                         & Difference                        & Accuracy                         & Difference                       & Accuracy & Difference                  & Accuracy & Difference                  \\
		\midrule
		No Dropout             & 93.2700                          & -                                 & 93.7067                          & -                                & 70.6200  & -                           & 71.8833  & -                           \\
		\midrule
		Guideline \ref{gui:my} & 93.4600                          & (\textcolor{blue}{+0.1900})       & 93.9500                          & (\textcolor{blue}{+0.2433})      & 71.4867  & (\textcolor{blue}{+0.8667}) & 73.1133  & (\textcolor{blue}{+1.2300}) \\
		\bottomrule
	\end{tabular}
\end{table*}

\begin{table*}[ht!]
	\centering
	\caption{Test accuracy on Caltech-101 and Oxford-IIIT Pet datasets.}
	\label{tab:petcal}
	\begin{tabular}{l|rr|rr|rr|rr}
		\toprule
		                       & \multicolumn{4}{c|}{Caltech-101}  & \multicolumn{4}{c}{Oxford-IIIT Pet}                                                                                                                                                           \\
		                       & \multicolumn{2}{c|}{PreResNet-50} & \multicolumn{2}{c|}{PreResNet-101}  & \multicolumn{2}{c|}{PreResNet-50} & \multicolumn{2}{c}{PreResNet-101}                                                                                   \\
		                       & Accuracy                          & Difference                          & Accuracy                          & Difference                        & Accuracy & Difference                  & Accuracy & Difference                  \\
		\midrule
		No Dropout             & 83.3212                           & -                                   & 83.8074                           & -                                 & 83.8147  & -                           & 84.5668  & -                           \\
		\midrule
		Guideline \ref{gui:my} & 83.7831                           & (\textcolor{blue}{+0.4620})         & 84.1964                           & (\textcolor{blue}{+0.3890})       & 83.9049  & (\textcolor{blue}{+0.0903}) & 85.5897  & (\textcolor{blue}{+1.0229}) \\
		\bottomrule
	\end{tabular}
\end{table*}

\subsection{Dropout in Residual Block}
\paragraph{CIFAR Dataset}
We conducted experiments to observe the performance differences due to the dropout position. First, we compared the performance of PreResNet trained without and with dropout at one of (P0, $\cdots$, P7). We trained PreResNet-$\{50,\ 110\}$ on a multi-class classification task using the CIFAR-$\{10,\ 100\}$ datasets \citep{Krizhevsky09learningmultiple}. See the Appendix for details such as the hyperparameters used. An average of three runs was reported for each result (Table \ref{tab:cifar}).

The experimental results were in agreement with our claims. The greatest accuracy was observed when dropout was applied at P5 or P6, confirming the validity of Guideline \ref{gui:my}. Applying dropout at (P5, P6, or P7) resulted in improved accuracy, whereas applying dropout at (P0, P1, P2, P3, or P4) decreased accuracy. This observation implies that the placement of dropout after the second BN matters more than whether it is applied after weight or ReLU. That is, when dropout was applied before the second BN, neither PreDropout nor PostDropout improved the accuracy. An explanation for this phenomenon requires Propositions \ref{pro:zeroth} and \ref{pro:lth}, and is unique to our analysis compared to the existing literature.

\paragraph{Other Conditions}
We further validated Guideline \ref{gui:my} using other experimental setups. To test different weight conditions, we varied the weight decay from $10^{-4}$ to $10^{-3}$ or $10^{-5}$. We also experimented with PreResNet using a bottleneck block, which had three [BN--ReLU--Weight] pipelines, unlike the basic block. In this case, to follow Guideline \ref{gui:my}, we applied dropout after the third ReLU but before the third weight layer. We also tested our guideline with ELU \citep{DBLP:journals/corr/ClevertUH15} to replace ReLU. For the four experimental setups, we observed an improved accuracy for CIFAR-10 and PreResNet-50 (Table \ref{tab:condition}). Note that applying dropout did not always improve and could actually degrade the performance (Table \ref{tab:cifar}); however, applying dropout in accordance with Guideline \ref{gui:my} consistently and successfully improved performance.

\paragraph{Other Models}
In addition, we experimented with Guideline \ref{gui:my} using the original ResNetV1, whose residual branch had two [Weight--BN--ReLU] pipelines. In this case, to follow Guideline \ref{gui:my}, we applied dropout at the end of the residual branch. We used the ResNetV1-$\{50,\ 110\}$ and CIFAR-$\{10,\ 100\}$ datasets. Again, we observed improved accuracy from dropout using Guideline \ref{gui:my} (Table \ref{tab:resnetv1}).

\paragraph{Other Datasets}
We further validated our claim using other datasets. Two datasets were targeted, Caltech-101 and Oxford-IIIT Pet \citep{DBLP:journals/cviu/Fei-FeiFP07,DBLP:conf/cvpr/ParkhiVZJ12}. We used PreResNet-$\{50,\ 101\}$ with bottleneck block. See the Appendix for details such as the hyperparameters used. On the two datasets and two PreResNets, we observed that applying dropout following Guideline \ref{gui:my} improved the test accuracy (Table \ref{tab:petcal}).

\subsection{Dropout in Head}

\begin{table}[t!]
	\centering
	\caption{Experimental results on dropout at the head. ``P'' represents PreResNet. $^*$ indicates applying dropout following Guideline \ref{gui:head}.}
	\label{tab:petcalhead}
	\begin{tabular}{l|l|r|r|r}
		\toprule
		\textbf{Dataset}     & \textbf{Model} & \textbf{No Dropout} & \textbf{H4}$^*$ & \textbf{H5} \\
		\midrule
		\multirow{2}{*}{Cal} & P-50           & 83.321              & 83.978          & 83.005      \\
		                     & P-101          & 83.807              & 84.999          & 83.662      \\
		\midrule
		\multirow{2}{*}{Pet} & P-50           & 83.815              & 84.988          & 84.416      \\
		                     & P-101          & 84.567              & 85.229          & 85.259      \\
		\bottomrule
	\end{tabular}
\end{table}

\paragraph{Caltech-101 and Oxford-IIIT Pet} We experimented with dropout in the head. We compared three cases: training without dropout and with dropout at one of (H4, H5). We trained PreResNet-$\{50,\ 101\}$ on the Caltech-101 and Oxford-IIIT Pet datasets. The average of three runs was reported for each result (Table \ref{tab:petcalhead}). We observed that applying dropout at H4 demonstrated greater accuracy than H5. Applying dropout at H5, which is the current consensus in existing studies, improved the accuracy on Oxford-IIIT Pet, yet decreased accuracy on Caltech-101.

\begin{table}[t!]
	\centering
	\caption{Top-1 accuracy on ImageNet.}
	\label{tab:inhead}
	\begin{tabular}{l|r|r|r}
		\toprule
		\textbf{Model}    & \textbf{No Dropout} & \textbf{H4}$^*$ & \textbf{H5} \\
		\midrule
		MobileNetV2 (1.4) & 75.714              & 75.820          & 75.718      \\
		EfficientNet-B0   & 77.156              & 77.240          & 76.976      \\
		ResNet-50         & 78.834              & 78.932          & 78.546      \\
		DenseNet-169      & 79.066              & 79.152          & 79.036      \\
		\bottomrule
	\end{tabular}
\end{table}

\paragraph{ImageNet} We further validated our claim using another dataset and models. We targeted ImageNet, a widely used large-scale dataset. See the Appendix for details such as the hyperparameters used. Because our analysis on the head is applicable to any model that employs GAP, we targeted other models: MobileNetV2, EfficientNet, ResNet, and DenseNet. We observed that applying dropout at H4 consistently improved Top-1 accuracy on the ImageNet dataset (Table \ref{tab:inhead}). Note that MobileNetV2 and EfficientNet originally employed dropout at H5; however, we found that it could only marginally influence or even decrease accuracy.

\section{Conclusion}\label{sec:con}
In this study, we investigated the correct position for applying dropout. We demonstrated that the dropout position influences the variance inconsistency and sought the best position that provides an inconsistency ratio close to one. By analyzing the theoretical properties of the residual networks, we discovered the correct position to apply dropout was after the last BN but before the last weight layer. In several experiments, we observed increased and decreased accuracy depending on the position of dropout, explaining the reason for the performance change using our analysis. In addition, we provided a guideline on applying dropout at the head and validated the improved performance through experiments. We hope that these findings will help practitioners understand and benefit from dropouts.

\bibliography{mybib}

\begin{thebibliography}{34}
\providecommand{\natexlab}[1]{#1}
\providecommand{\url}[1]{\texttt{#1}}
\expandafter\ifx\csname urlstyle\endcsname\relax
  \providecommand{\doi}[1]{doi: #1}\else
  \providecommand{\doi}{doi: \begingroup \urlstyle{rm}\Url}\fi

\bibitem[Bello et~al.(2021)Bello, Fedus, Du, Cubuk, Srinivas, Lin, Shlens, and
  Zoph]{DBLP:conf/nips/BelloFDCSLSZ21}
Irwan Bello, William Fedus, Xianzhi Du, Ekin~Dogus Cubuk, Aravind Srinivas,
  Tsung{-}Yi Lin, Jonathon Shlens, and Barret Zoph.
\newblock {Revisiting ResNets: Improved Training and Scaling Strategies}.
\newblock In \emph{NeurIPS}, 2021.

\bibitem[Brock et~al.(2021)Brock, De, and Smith]{DBLP:conf/iclr/BrockDS21}
Andrew Brock, Soham De, and Samuel~L. Smith.
\newblock {Characterizing signal propagation to close the performance gap in
  unnormalized ResNets}.
\newblock In \emph{{ICLR}}, 2021.

\bibitem[Cai et~al.(2019)Cai, Gao, Zhang, Wang, Chen, and
  Ooi]{DBLP:journals/corr/abs-1904-03392}
Shaofeng Cai, Jinyang Gao, Meihui Zhang, Wei Wang, Gang Chen, and Beng~Chin
  Ooi.
\newblock {Effective and Efficient Dropout for Deep Convolutional Neural
  Networks}.
\newblock \emph{CoRR}, abs/1904.03392, 2019.

\bibitem[Castro et~al.(2021)Castro, Souto, Ogasawara, Porto, and
  Bezerra]{DBLP:journals/ijon/CastroSOP021}
Rafaela Castro, Yania~Molina Souto, Eduardo~S. Ogasawara, F{\'{a}}bio Porto,
  and Eduardo Bezerra.
\newblock {STConvS2S: Spatiotemporal Convolutional Sequence to Sequence Network
  for weather forecasting}.
\newblock \emph{Neurocomputing}, 2021.

\bibitem[Clevert et~al.(2016)Clevert, Unterthiner, and
  Hochreiter]{DBLP:journals/corr/ClevertUH15}
Djork{-}Arn{\'{e}} Clevert, Thomas Unterthiner, and Sepp Hochreiter.
\newblock {Fast and Accurate Deep Network Learning by Exponential Linear Units
  (ELUs)}.
\newblock In \emph{{ICLR}}, 2016.

\bibitem[De and Smith(2020)]{DBLP:conf/nips/DeS20}
Soham De and Samuel~L. Smith.
\newblock {Batch Normalization Biases Residual Blocks Towards the Identity
  Function in Deep Networks}.
\newblock In \emph{NeurIPS}, 2020.

\bibitem[Fei{-}Fei et~al.(2007)Fei{-}Fei, Fergus, and
  Perona]{DBLP:journals/cviu/Fei-FeiFP07}
Li~Fei{-}Fei, Robert Fergus, and Pietro Perona.
\newblock {Learning generative visual models from few training examples: An
  incremental Bayesian approach tested on 101 object categories}.
\newblock \emph{Comput. Vis. Image Underst.}, 2007.

\bibitem[Ghiasi et~al.(2018)Ghiasi, Lin, and Le]{DBLP:conf/nips/GhiasiLL18}
Golnaz Ghiasi, Tsung{-}Yi Lin, and Quoc~V. Le.
\newblock {DropBlock: {A} regularization method for convolutional networks}.
\newblock In \emph{NeurIPS}, 2018.

\bibitem[He et~al.(2015)He, Zhang, Ren, and Sun]{DBLP:conf/iccv/HeZRS15}
Kaiming He, Xiangyu Zhang, Shaoqing Ren, and Jian Sun.
\newblock {Delving Deep into Rectifiers: Surpassing Human-Level Performance on
  ImageNet Classification}.
\newblock In \emph{{ICCV}}, 2015.

\bibitem[He et~al.(2016{\natexlab{a}})He, Zhang, Ren, and
  Sun]{DBLP:conf/cvpr/HeZRS16}
Kaiming He, Xiangyu Zhang, Shaoqing Ren, and Jian Sun.
\newblock Deep residual learning for image recognition.
\newblock In \emph{{CVPR}}, 2016{\natexlab{a}}.

\bibitem[He et~al.(2016{\natexlab{b}})He, Zhang, Ren, and
  Sun]{DBLP:conf/eccv/HeZRS16}
Kaiming He, Xiangyu Zhang, Shaoqing Ren, and Jian Sun.
\newblock {Identity Mappings in Deep Residual Networks}.
\newblock In \emph{{ECCV}}, 2016{\natexlab{b}}.

\bibitem[Ioffe and Szegedy(2015)]{DBLP:conf/icml/IoffeS15}
Sergey Ioffe and Christian Szegedy.
\newblock {Batch Normalization: Accelerating Deep Network Training by Reducing
  Internal Covariate Shift}.
\newblock In \emph{{ICML}}, 2015.

\bibitem[Isola et~al.(2017)Isola, Zhu, Zhou, and
  Efros]{DBLP:conf/cvpr/IsolaZZE17}
Phillip Isola, Jun{-}Yan Zhu, Tinghui Zhou, and Alexei~A. Efros.
\newblock {Image-to-Image Translation with Conditional Adversarial Networks}.
\newblock In \emph{{CVPR}}, 2017.

\bibitem[Krizhevsky(2009)]{Krizhevsky09learningmultiple}
Alex Krizhevsky.
\newblock {Learning multiple layers of features from tiny images}.
\newblock Technical report, 2009.

\bibitem[Li et~al.(2019)Li, Chen, Hu, and Yang]{DBLP:conf/cvpr/0028C0019}
Xiang Li, Shuo Chen, Xiaolin Hu, and Jian Yang.
\newblock {Understanding the Disharmony Between Dropout and Batch Normalization
  by Variance Shift}.
\newblock In \emph{{CVPR}}, 2019.

\bibitem[Lim et~al.(2016)Lim, Gehre, and Kobbelt]{DBLP:journals/cgf/LimGK16}
Isaak Lim, Anne Gehre, and Leif Kobbelt.
\newblock {Identifying Style of 3D Shapes using Deep Metric Learning}.
\newblock \emph{Comput. Graph. Forum}, 2016.

\bibitem[Liu et~al.(2020)Liu, Xu, and Zhang]{DBLP:journals/corr/abs-2006-15619}
Brian Liu, Xianchao Xu, and Yu~Zhang.
\newblock {Offline Handwritten Chinese Text Recognition with Convolutional
  Neural Networks}.
\newblock \emph{CoRR}, abs/2006.15619, 2020.

\bibitem[Parkhi et~al.(2012)Parkhi, Vedaldi, Zisserman, and
  Jawahar]{DBLP:conf/cvpr/ParkhiVZJ12}
Omkar~M. Parkhi, Andrea Vedaldi, Andrew Zisserman, and C.~V. Jawahar.
\newblock {Cats and dogs}.
\newblock In \emph{{CVPR}}, 2012.

\bibitem[Pavllo et~al.(2019)Pavllo, Feichtenhofer, Grangier, and
  Auli]{DBLP:conf/cvpr/PavlloFGA19}
Dario Pavllo, Christoph Feichtenhofer, David Grangier, and Michael Auli.
\newblock {3D Human Pose Estimation in Video With Temporal Convolutions and
  Semi-Supervised Training}.
\newblock In \emph{{CVPR}}, 2019.

\bibitem[Pham and Le(2021)]{DBLP:conf/aaai/PhamL21}
Hieu Pham and Quoc~V. Le.
\newblock {AutoDropout: Learning Dropout Patterns to Regularize Deep Networks}.
\newblock In \emph{{AAAI}}, 2021.

\bibitem[Qi et~al.(2017)Qi, Yi, Su, and Guibas]{DBLP:conf/nips/QiYSG17}
Charles~Ruizhongtai Qi, Li~Yi, Hao Su, and Leonidas~J. Guibas.
\newblock {PointNet++: Deep Hierarchical Feature Learning on Point Sets in a
  Metric Space}.
\newblock In \emph{{NIPS}}, 2017.

\bibitem[Ravi and Larochelle(2017)]{DBLP:conf/iclr/RaviL17}
Sachin Ravi and Hugo Larochelle.
\newblock {Optimization as a Model for Few-Shot Learning}.
\newblock In \emph{{ICLR}}, 2017.

\bibitem[Romera et~al.(2018)Romera, Alvarez, Bergasa, and
  Arroyo]{DBLP:journals/tits/RomeraABA18}
Eduardo Romera, Jose~M. Alvarez, Luis~Miguel Bergasa, and Roberto Arroyo.
\newblock {ERFNet: Efficient Residual Factorized ConvNet for Real-Time Semantic
  Segmentation}.
\newblock \emph{{IEEE} Trans. Intell. Transp. Syst.}, 2018.

\bibitem[Sandler et~al.(2018)Sandler, Howard, Zhu, Zhmoginov, and
  Chen]{DBLP:conf/cvpr/SandlerHZZC18}
Mark Sandler, Andrew~G. Howard, Menglong Zhu, Andrey Zhmoginov, and
  Liang{-}Chieh Chen.
\newblock {MobileNetV2: Inverted Residuals and Linear Bottlenecks}.
\newblock In \emph{{CVPR}}, 2018.

\bibitem[Srivastava et~al.(2014)Srivastava, Hinton, Krizhevsky, Sutskever, and
  Salakhutdinov]{DBLP:journals/jmlr/SrivastavaHKSS14}
Nitish Srivastava, Geoffrey~E. Hinton, Alex Krizhevsky, Ilya Sutskever, and
  Ruslan Salakhutdinov.
\newblock {Dropout: a simple way to prevent neural networks from overfitting}.
\newblock \emph{J. Mach. Learn. Res.}, 2014.

\bibitem[Szegedy et~al.(2016)Szegedy, Vanhoucke, Ioffe, Shlens, and
  Wojna]{DBLP:conf/cvpr/SzegedyVISW16}
Christian Szegedy, Vincent Vanhoucke, Sergey Ioffe, Jonathon Shlens, and
  Zbigniew Wojna.
\newblock {Rethinking the Inception Architecture for Computer Vision}.
\newblock In \emph{{CVPR}}, 2016.

\bibitem[Szegedy et~al.(2017)Szegedy, Ioffe, Vanhoucke, and
  Alemi]{DBLP:conf/aaai/SzegedyIVA17}
Christian Szegedy, Sergey Ioffe, Vincent Vanhoucke, and Alexander~A. Alemi.
\newblock {Inception-v4, Inception-ResNet and the Impact of Residual
  Connections on Learning}.
\newblock In \emph{{AAAI}}, 2017.

\bibitem[Tan and Le(2019)]{DBLP:conf/icml/TanL19}
Mingxing Tan and Quoc~V. Le.
\newblock {EfficientNet: Rethinking Model Scaling for Convolutional Neural
  Networks}.
\newblock In \emph{{ICML}}, 2019.

\bibitem[Tan and Le(2021)]{DBLP:conf/icml/TanL21}
Mingxing Tan and Quoc~V. Le.
\newblock {EfficientNetV2: Smaller Models and Faster Training}.
\newblock In \emph{{ICML}}, 2021.

\bibitem[Tan et~al.(2019)Tan, Chen, Pang, Vasudevan, Sandler, Howard, and
  Le]{DBLP:conf/cvpr/TanCPVSHL19}
Mingxing Tan, Bo~Chen, Ruoming Pang, Vijay Vasudevan, Mark Sandler, Andrew
  Howard, and Quoc~V. Le.
\newblock {MnasNet: Platform-Aware Neural Architecture Search for Mobile}.
\newblock In \emph{{CVPR}}, 2019.

\bibitem[Yan et~al.(2018)Yan, Xiong, and Lin]{DBLP:conf/aaai/YanXL18}
Sijie Yan, Yuanjun Xiong, and Dahua Lin.
\newblock {Spatial Temporal Graph Convolutional Networks for Skeleton-Based
  Action Recognition}.
\newblock In \emph{{AAAI}}, 2018.

\bibitem[Zagoruyko and Komodakis(2016)]{DBLP:conf/bmvc/ZagoruykoK16}
Sergey Zagoruyko and Nikos Komodakis.
\newblock {Wide Residual Networks}.
\newblock In \emph{{BMVC}}, 2016.

\bibitem[Zhan et~al.(2020)Zhan, Xie, Liu, Ong, and
  Loy]{DBLP:conf/cvpr/ZhanX0OL20}
Xiaohang Zhan, Jiahao Xie, Ziwei Liu, Yew{-}Soon Ong, and Chen~Change Loy.
\newblock {Online Deep Clustering for Unsupervised Representation Learning}.
\newblock In \emph{{CVPR}}, 2020.

\bibitem[Zoph et~al.(2018)Zoph, Vasudevan, Shlens, and
  Le]{DBLP:conf/cvpr/ZophVSL18}
Barret Zoph, Vijay Vasudevan, Jonathon Shlens, and Quoc~V. Le.
\newblock {Learning Transferable Architectures for Scalable Image Recognition}.
\newblock In \emph{{CVPR}}, 2018.

\end{thebibliography}
\end{document}